\documentclass[letterpaper]{article} 
\usepackage{aaai26}  
\usepackage{times}  
\usepackage{helvet}  
\usepackage{courier}  
\usepackage[hyphens]{url}  
\usepackage{graphicx} 
\usepackage{natbib}  
\usepackage{caption} 
\usepackage{algorithm}
\usepackage{algorithmic}
\usepackage{amsmath}
\usepackage{amssymb}
\usepackage{amsthm}
\usepackage{booktabs}
\usepackage{multirow}
\usepackage{makecell}
\usepackage{xcolor}
\usepackage[T1]{fontenc}
\newtheorem{theorem}{Theorem}
\newtheorem{assumption}{Assumption}
\newtheorem{lemma}{Lemma}

\DeclareMathOperator*{\argmin}{arg\,min}
\usepackage{newfloat}
\usepackage{listings}
\DeclareCaptionStyle{ruled}{labelfont=normalfont,labelsep=colon,strut=off} 
\floatstyle{ruled}
\newfloat{listing}{tb}{lst}{}
\floatname{listing}{Listing}
\title{Bridging Dynamics Gaps via Diffusion Schrödinger Bridge for Cross-Domain Reinforcement Learning}
\author{
 Hanping Zhang, \quad Yuhong Guo
}
\affiliations{
    School of Computer Science, Carleton University, Ottawa, Canada	

}

\begin{document}

\maketitle

\begin{abstract}
Cross-domain reinforcement learning (RL) aims to learn transferable policies under dynamics shifts between source and target domains. 
A key challenge lies in the lack of target-domain environment 
	interaction and reward supervision, 
which prevents direct policy learning. 
To address this challenge, we propose Bridging Dynamics Gaps for Cross-Domain Reinforcement Learning (BDGxRL), a novel framework that leverages Diffusion Schrödinger Bridge (DSB) 
	to align source transitions with target-domain dynamics encoded in 
	offline demonstrations. 
Moreover, we introduce a reward modulation mechanism that estimates rewards based on state transitions, 
applying to DSB-aligned samples to ensure consistency between rewards and target-domain dynamics. 
BDGxRL performs target-oriented policy learning entirely within the source domain, without 
	access to the target environment or its rewards. 
Experiments on MuJoCo cross-domain benchmarks demonstrate that BDGxRL outperforms state-of-the-art baselines and shows strong adaptability under transition dynamics shifts.
\end{abstract}

\section{Introduction}
Reinforcement learning (RL), 
a cornerstone of modern sequential decision-making,
has been successfully applied across a wide range of domains, including robotics~\citep{kober2013reinforcement}, game playing~\citep{berner2019dota}, recommender systems~\citep{afsar2022reinforcement}, and autonomous driving~\citep{aradi2020survey}. By interacting with an environment to maximize long-term cumulative rewards, RL agents can learn complex policies from experience. However, directly training RL agents in real-world scenarios is often impractical due to limitations in data availability, environment accessibility, and retraining flexibility. This challenge is particularly evident in cross-domain settings, 
such as transferring policies between simulation and the real world, or adapting quickly between related tasks or environments. Bridging 
the cross-domain gap is critical to improving the real-world applicability of RL.

To tackle this issue, cross-domain adaptation in RL has attracted increasing attention. In such settings, agents are typically trained in a source domain (e.g., a simulator) and then deployed in a target domain (e.g., the real world). In many real-world applications, the source and target domains often share the same state and action spaces, since the simulated agents are typically designed to mimic their physical counterparts. However, the transition dynamics between the two domains often differ due to subtle mismatches in physical properties such as gravity, mass, or friction. These discrepancies, known as dynamics gaps, can result in significant degradation in policy performance when transferring from simulation to reality.

Addressing these dynamics gaps is non-trivial, especially when direct interaction with the real-world environment is unavailable or restricted. A promising alternative is to collect limited offline expert demonstrations from the target domain. However, these demonstrations are typically sparse and lack reward annotations, making it difficult to directly apply conventional RL objectives. Moreover, reusing the reward function from the source domain is often inappropriate, as the dynamics gap may induce a mismatch between the source and target reward realizations, even when the nominal reward function remains unchanged. This reward inconsistency further complicates the challenge of learning transferable policies. As a result, effective cross-domain adaptation methods must address both the divergence in transition dynamics and the scarcity of explicit supervision in the target domain, including the absence of reward signals.

To bridge
the dynamics gap in cross-domain RL, we consider the Diffusion Schrödinger Bridge (DSB)~\citep{de2021diffusion}, a recently developed probabilistic framework for distribution alignment. DSB formulates the interpolation between two probability distributions as a stochastic optimal transport (OT) problem, solved via a time-dependent diffusion process. Unlike traditional OT~\citep{peyre2019computational}, which often requires explicit coupling or cost functions, DSB learns a continuous flow that transforms samples from a source distribution to a target distribution without requiring paired data.
Originally applied in generative modeling tasks such as unpaired image translation and distribution morphing, DSB has shown strong capabilities in capturing complex structural differences between domains. 

Motivated by DSB's
strengths, we explore its potential in cross-domain RL for transition dynamics adaptation. Specifically, by learning a DSB process that transports source-domain transitions to match the distribution of offline expert transitions in the target domain, we can generate aligned trajectories that better reflect the target environment. This alignment is achieved without requiring online interactions or reward annotations in the target domain, making DSB particularly suitable for offline cross-domain RL settings where supervision is limited.

To address the overall challenge of cross-domain adaptation in RL, 
we propose a novel framework named Bridging Dynamics Gaps for Cross-Domain Reinforcement Learning (BDGxRL), 
which leverages DSB to tackle dynamics gaps without requiring access to the target environment. Our method consists of three key components: First, we learn a DSB-based transition bridge using only offline expert demonstrations from the target domain and trajectories sampled from the online source environment. 
This enables the adaptation of transition dynamics in a purely offline manner. 
Second, we perform reward modulation by training an action-independent reward model on the source domain, which estimates rewards solely based on the state transition pairs. 
This reward model is 
modulated using the DSB-transformed transitions to the target domain, 
mitigating reward mismatch caused by dynamics differences 
in the absence of target-domain rewards. 
Third, the learned transition DSB is used to transform source-domain trajectories into target-style transitions, upon which the modulated reward is applied. This allows us to learn a target-oriented policy entirely within the source domain, without any online interaction or direct data collection with rewards from the target environment. We evaluate BDGxRL on a set of cross-domain benchmarks built upon the MuJoCo simulation suite~\citep{todorov2012mujoco}, and compare it with several state-of-the-art cross-domain RL methods. Experimental results demonstrate that BDGxRL consistently outperforms existing baselines, achieving superior adaptability 
under challenging dynamics shifts.
Our contributions are summarized as follows:
\begin{itemize}
\item We propose BDGxRL, a novel framework that enables learning a target-oriented policy 
	in the source environment,
	addressing dynamics gaps 
		in cross-domain RL. 
\item We are the first to introduce the DSB into cross-domain RL, enabling transition dynamics transformation from the source domain to align with the target domain.
\item We identify 
	that changes in transition dynamics can induce inconsistencies in reward functions. 
		To mitigate this, we introduce a reward modulation mechanism that supplements missing target-domain rewards.
\item Experimental results on multiple physics-based simulation benchmarks demonstrate that BDGxRL consistently outperforms state-of-the-art cross-domain RL baselines. 
\end{itemize}
\section{Related Work}
\subsection{Cross-Domain Reinforcement Learning}
Cross-domain RL 
dates back to the early applications of domain adaptation (DA) in RL~\citep{lazaric2012transfer}. With the rise of deep reinforcement learning, \citet{higgins2017darla} were among the first to explicitly introduce domain adaptation to RL. 
They systematically identified the key factors that may differ across domains: observations, actions, rewards, and transitions, and proposed the DisentAngled Representation Learning Agent (DARLA) to handle such domain shifts. 
Building upon DARLA, \citet{xing2021domain} proposed Latent Unified State Representation (LUSR), which learns a shared latent space across domains to mitigate observation-level discrepancies, enabling policy transfer without explicitly addressing dynamics mismatch.

A significant shift in focus came with \citet{eysenbach2020off}, who introduced Domain Adaptation with Rewards from Classifiers (DARC), 
a modern approach to consider the off-dynamics setting.
DARC trains a classifier to distinguish source-domain transitions from target-domain ones, and modifies the reward function to penalize source-domain transitions. 
\citet{liu2022dara} extended this idea to the offline setting with Dynamics-Aware Reward Augmentation (DARA). Instead of distinguishing transitions based solely on rewards, DARA learns a classifier to identify domain discrepancies over full $(s,a,s')$ tuples and modifies rewards to penalize transitions that do not appear plausible in the target domain.
Further advancing this line of work, \citet{guo2024off} proposed Domain Adaptation and Reward-Augmented Imitation Learning (DARAIL), which incorporates adversarial imitation learning from observations to align source and target behaviors. They also introduce a Reward-Augmented Estimator (RAE) to refine reward modification based on inferred dynamics mismatch.
More recently, \citet{niu2024xted} proposed Cross-Domain Trajectory EDiting (xTED), which employs a diffusion model to transform source domain trajectories into target-like ones. By injecting and subsequently denoising noise, xTED learns to generate plausible target-domain trajectories for policy training.
\citet{lyu2024cross} introduced 
a cross-domain policy adaptation method by capturing representation mismatch,
which learns two encoders: one for state-action pairs and one for next states, trained exclusively on the target domain. 
Transitions with large representation deviations under these encoders
are penalized via reward modification, effectively filtering out implausible transitions. 
All these methods share the off-dynamics setting, where the dynamics differ between domains, and source domain data/environment is abundant while access to the target domain is limited. 
Our work focuses 
on learning a target-oriented policy entirely from an online source domain, 
assuming only limited demonstrations are available from the target domain.

Several recent works adopt related but slightly different settings. For example, some assume either offline source or offline target datasets while still requiring limited target domain interactions, rather than purely learning from the source domain.
\citet{xu2023cross} focus on the online source domain 
but allow a few online interactions in the target domain. They propose Value-Guided Data Filtering (VGDF), which leverages value function consistency across domains to select source domain transitions whose estimated values 
closely match those in the target domain.
\citet{niu2022trust} study a Sim2Real setting where both limited real-world data and an imperfect simulator are available. Their proposed Dynamics-Aware Hybrid Offline-and-Online Reinforcement Learning (H2O) framework jointly learns from real-world offline data and simulator-based online rollouts. H2O adaptively penalizes simulated transitions with large dynamics gaps through a dynamics-aware Q-function adjustment. 
\subsection{Diffusion Schrödinger Bridge}
Diffusion models~\citep{ho2020denoising, song2020denoising} have recently gained widespread recognition for their powerful generative capabilities 
in domains such as image generation~\citep{rombach2022high}. 
In addition, 
their conditional variants~\citep{song2020score, zhang2023text} enable controlled generation with flexible conditioning inputs.
Moreover, 
the particularly important challenge of data translation, especially in scenarios involving two large unpaired datasets, necessitates extending diffusion models beyond unconditional generation.

The Optimal Transport (OT) problem~\citep{peyre2019computational, santambrogio2015optimal}, which seeks the most efficient way to map one probability distribution to another, provides a fundamental framework for addressing data translation. 
\citet{li2023dpm} integrate OT theory into the diffusion process by computing semi-discrete OT maps during denoising, thereby improving both the sampling efficiency and quality.
To address unpaired or partially paired datasets, 
\citet{gu2023optimal} introduce optimal transport-guided conditional score-based diffusion models.
Their method couples unpaired data distributions using OT and integrates this coupling into 
conditional score-based diffusion models,
achieving strong image translation.

\citet{de2021diffusion} were the first to apply the concept of the Schrödinger Bridge (SB)~\citep{leonard2013survey} to diffusion models, offering a probabilistic framework to solve OT problems by finding the optimal stochastic process, e.g., Brownian motion, connecting two distributions. They introduce the Diffusion Schrödinger Bridge (DSB), enabling data translation between distinct domains, and solve it via Iterative Proportional Fitting (IPF).
Subsequently, \citet{shi2023diffusion} propose Iterative Markovian Fitting (IMF) and the associated Diffusion Schrödinger Bridge Matching (DSBM) algorithm to efficiently compute the SB solution.
Building on these, \citet{de2024schrodinger} introduce improved variants: $\alpha$-IMF and $\alpha$-DSBM, designed for more challenging unpaired data translation tasks, achieving state-of-the-art performance. \citet{somnath2023aligned} propose SBALIGN, which leverages aligned datasets 
to solve DSB more efficiently. Meanwhile, \citet{kim2024fast} introduce Data Ensemble (DE), a method for enhancing the neural backbone of DSB.
Despite DSB’s growing success in unpaired image data translation, its potential applications in other areas remain underexplored. In this work, we investigate the use of DSB 
to align the transition dynamics between the source and target domains, facilitating 
effective policy transfer across environments with mismatched dynamics.

\begin{figure*}[t]
\centering
\includegraphics[width=0.80\textwidth]{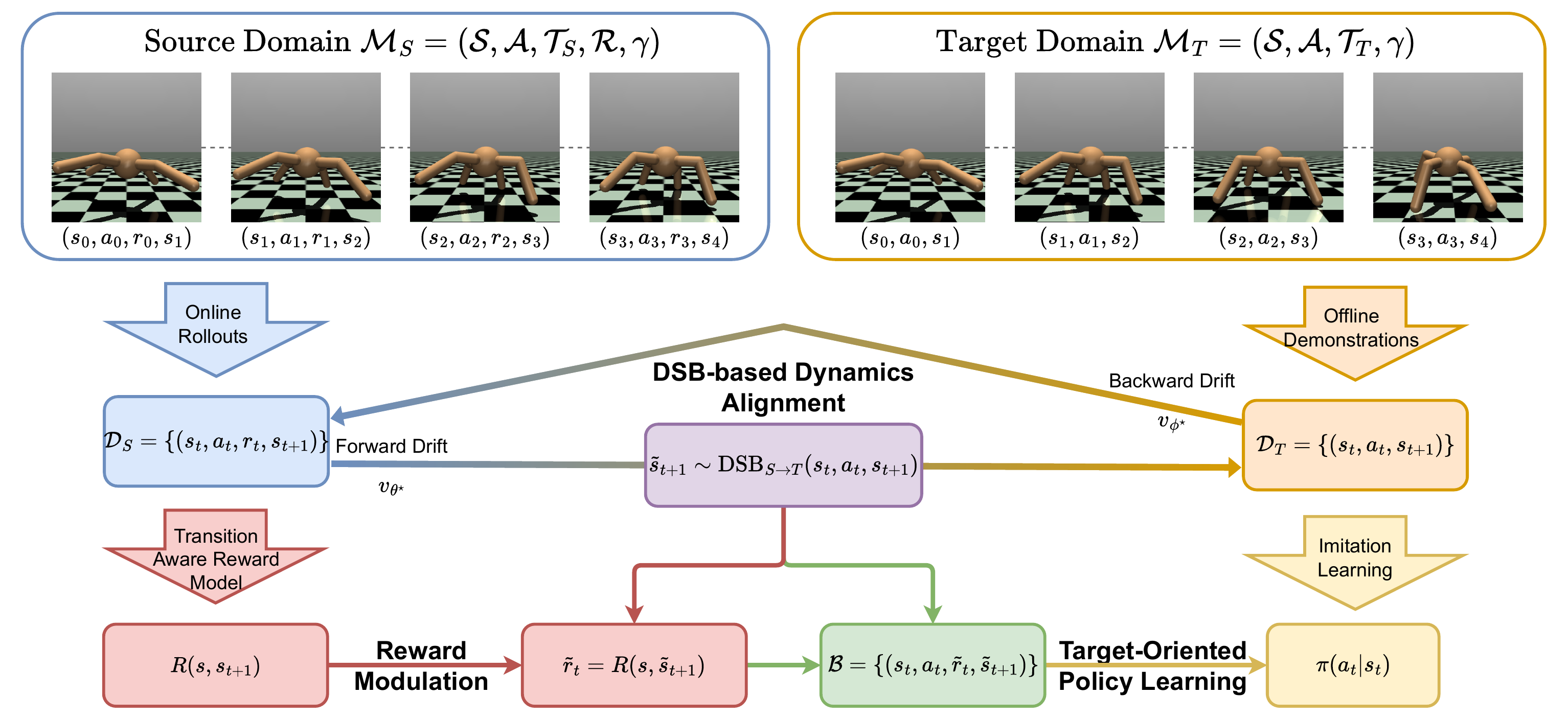}
\caption{
Overview of the proposed BDGxRL framework. The agent first collects a dataset $\mathcal{D}_S$ from the source environment,
which is used to train a transition-aware reward model $R(s_t,s_{t+1})$. Together with offline target demonstrations $\mathcal{D}_T$, it also trains a DSB model for dynamics alignment. During online interactions, source transitions are translated into target-style transitions via $\tilde{s}_{t+1}\sim\mathrm{DSB}(s_t,a_t,s_{t+1})$ to mitigate dynamics mismatch. The modulated reward $\tilde{r}_t=R(s_t,\tilde{s}_{t+1})$ is then used to learn a target-oriented policy entirely within the source domain, initialized via imitation from $\mathcal{D}_T$.
}
\label{fig:model}
\end{figure*}
%
\section{Preliminaries}
In this section, we provide a brief 
introduction
on the Diffusion Schrödinger Bridge (DSB) problem and its practical realization via the Iterative Markov Fitting (IMF) procedure. These tools are foundational to our approach for translating transition dynamics across domains.
\paragraph{Diffusion Schrödinger Bridge} 
The Schrödinger Bridge (SB) problem seeks a stochastic process that transports an initial distribution $\Pi_0$ to a terminal distribution $\Pi_1$ over time, while remaining as close as possible, in the sense of Kullback–Leibler (KL) divergence, to a given reference process $\mathbb{Q}$, typically a Brownian motion.

In Diffusion Schrödinger Bridge (DSB), the reference process $\mathbb{Q}$ is taken to be a diffusion process with known time-forward dynamics. Given marginals $\Pi_0$ and $\Pi_1$, the goal is to find an optimal path measure $\mathbb{P}^\star$ that minimizes the KL divergence from the reference diffusion process:
\begin{equation}
\mathbb{P}^* = \argmin_{\mathbb{P} \in \mathcal{M}(\Pi_0,\Pi_1)} 
\left\{ \mathrm{KL}(\mathbb{P} \,\|\, \mathbb{Q}) 
\;\middle|\; \mathbb{P}_0 = \Pi_0,\; \mathbb{P}_1 = \Pi_1 \right\}
\end{equation}
where $\mathcal{M}(\Pi_0,\Pi_1)$ denotes the set of Markov path measures with the prescribed initial and terminal marginals. Once the optimal bridge $\mathbb{P}^\star$ is obtained, it defines a stochastic transformation that couples $\Pi_0$ and $\Pi_1$, allowing us to effectively translate between the two distributions.
\paragraph{Iterative Markov Fitting}
Iterative Markov Fitting (IMF)~\citep{shi2023diffusion} 
is an efficient 
procedure for solving the SB problem. IMF introduces two alternating projection operators: the Markov projection $\mathrm{proj}_{\mathcal{M}}$~\citep{gyongy1986mimicking} and the reciprocal projection $\mathrm{proj}_{\mathcal{R}(\mathbb{Q})}$~\citep{leonard2014reciprocal}. Starting from an initial path measure, these projections are applied iteratively to generate a sequence of path measures $\mathbb{P}^n$ as follows:
\begin{equation}
\mathbb{P}^{2n+1} = \mathrm{proj}_{\mathcal{M}}(\mathbb{P}^{2n}), \quad
\mathbb{P}^{2n+2} = \mathrm{proj}_{\mathcal{R}(\mathbb{Q})}(\mathbb{P}^{2n+1}).
\end{equation}
This alternating procedure ensures that at every even step $2n$, the path measure $\mathbb{P}^{2n}$ satisfies the marginal constraints $\mathbb{P}^{2n}_0 = \Pi_0$ and $\mathbb{P}^{2n}_1 = \Pi_1$. With sufficient iterations, the sequence $\mathbb{P}^n$ converges to the optimal solution $\mathbb{P}^\star$ to the Schrödinger Bridge problem:
$\lim_{n\to\infty}\mathrm{KL}(\mathbb{P}^n \,\|\, \mathbb{P}^\star)=0.$
\section{Method}
\paragraph{Problem Setting}
Reinforcement learning (RL) is typically formulated as a Markov Decision Process (MDP), defined by the tuple $\mathcal{M} = (\mathcal{S}, \mathcal{A}, \mathcal{T}, \mathcal{R}, \gamma)$, where $\mathcal{S}$ is the state space, $\mathcal{A}$ is the action space, 
$\mathcal{T} : \mathcal{S} \times \mathcal{A} \rightarrow \mathcal{P}(\mathcal{S})$ 
represents the transition dynamics, $\mathcal{R} : \mathcal{S} \times \mathcal{A} \rightarrow \mathbb{R}$ is the reward function, and $\gamma \in (0, 1)$ is the discount factor.
The cross-domain RL setting involves two related MDPs: a source domain $\mathcal{M}_S$ and a target domain $\mathcal{M}_T$. These two domains are assumed to share the common state and action spaces, i.e., $\mathcal{S}_S = \mathcal{S}_T$ and $\mathcal{A}_S = \mathcal{A}_T$, but differ in their transition dynamics, with $\mathcal{T}_S \ne \mathcal{T}_T$, resulting in a dynamics gap. In the source domain $\mathcal{M}_S$, we assume access to an online environment that allows unlimited interactions and provides reward signals. This enables the agent to freely sample transitions $(s_t, a_t, r_t, s_{t+1})$ during training. In contrast, the target domain $\mathcal{M}_T$ is offline and reward-free. Specifically, we are only given a static 
dataset of offline trajectories $\mathcal{D}_T = \{(s_t, a_t, s_{t+1})\}$ collected in the target domain, without access to either interactive environment or ground-truth rewards.
The objective is to learn a policy $\pi(a|s)$ that performs well under the transition dynamics of $\mathcal{M}_T$, using only online data from the source domain and the offline dataset $\mathcal{D}_T$ from the target domain. This setting captures practical scenarios where the target domain is accessible only through offline expert demonstrations, and the reward function is difficult to specify or unavailable.

To address dynamics gaps across domains in this problem setting, 
we present a novel BDGxRL framework for target-oriented policy learning.  
It aligns the transition dynamics of the source and target domains based on 
the Diffusion Schrödinger Bridge (DSB). 
As the reward function in the target domain is unavailable, 
a reward modulation strategy is designed to adapt the reward model trained on the source domain to estimate rewards for DSB-aligned transitions. 
Building on dynamics alignment and reward modulation, 
the target-oriented policy learning framework is developed to enable policy optimization for the target domain 
in the source online environment, 
without requiring access to the target environment or rewards. 
We also provide a theoretical analysis on the performance bound of the learned target policy. 
An overview of our BDGxRL framework is presented in Figure~\ref{fig:model}.
\subsection{DSB-based Dynamics Alignment}
Bridging the dynamics gap between the source and target domains is particularly challenging in the absence of online interaction with the target environment. 
Given access only to a limited set of offline expert demonstrations from the target domain, denoted as $\mathcal{D}_T = \{(s_t, a_t, s_{t+1})\}$, 
while it is possible to learn a transition model $\mathcal{T}_T$ from $\mathcal{D}_T$, 
aligning or adapting the source-domain transition dynamics $\mathcal{T}_S$ to the target-domain dynamics remains difficult. The lack of paired transitions or any shared supervision between the domains further complicates this alignment task and renders direct policy transfer infeasible.

To address this problem, we draw inspiration from recent advances in unpaired data translation using the Diffusion Schrödinger Bridge (DSB), which has proven effective in tasks such as image domain adaptation. In our context, we formulate the source-domain transition dynamics as the source marginal distribution $\Pi_0 = \mathcal{T}_S(s_{t+1}|s_t, a_t)$ and the target-domain transition dynamics as the target marginal distribution $\Pi_1 = \mathcal{T}_T(s_{t+1}|s_t, a_t)$. Our goal is to learn a DSB that aligns these two distributions, thereby enabling 
one to translate transitions sampled from the source dynamics into transitions that are consistent with the target dynamics. Once this mapping is established, we can predict the 
 target-domain next state $s_{t+1}$ 
for any given source-domain state-action pair $(s_t, a_t)$ by translating the corresponding source transition to the target domain.
\paragraph{Transition Representation}
To model transition, 
we collect rollouts from the source environment to form a dataset $\mathcal{D}_S$. We then sample transitions $\mathbf{p}^S \sim \mathcal{D}_S$, where each transition represents a conditional distribution $p^S(s_{t+1} \mid s_t, a_t)$, and similarly, sample transitions $\mathbf{p}^T \sim \mathcal{D}_T$ from the target demonstrations. Each transition is represented as a concatenated vector of state, action, and next state:
\begin{equation}
\mathbf{p} = \begin{bmatrix}
s_t; & a_t; & s_{t+1}
\end{bmatrix}.
\end{equation}
This unified transition format allows us to treat transition samples as points in a shared representation space, enabling the application of DSB to learn a translation between source and target transition distributions.

\paragraph{DSB Training} 
To train the DSB model for transition dynamics alignment, we follow the Iterative Markov Fitting (IMF) procedure~\citep{shi2023diffusion} 
and adopt Brownian motion $\mathbb{B}$ as the reference process to simplify the formulation by avoiding additional drift terms. The objective is to learn two optimal drift functions, $v_{\theta^\star}$ and $v_{\phi^\star}$, corresponding to the forward and backward velocity fields that transport the source-domain transition distribution $\Pi_0$ to the target-domain distribution $\Pi_1$, and vice versa.

\paragraph{Forward Process}
In the forward direction, the process starts from a source-domain sample 
$\mathbf{p}_0^S \sim \Pi_0$ and evolves according to the following stochastic differential equation (SDE):
\begin{equation}
\mathrm{d}\mathbf{p}_k^S = v_{\theta^\star}(k, \mathbf{p}_k^S)\,\mathrm{d}k + \sigma_0\,\mathrm{d}\mathbf{B}_k,
\label{eq:forward}
\end{equation}
where $\sigma_0$ denotes a fixed diffusion magnitude and 
$\mathbf{B}_k$ is a standard Brownian motion with the same dimension as $\mathbf{p}_0^S$.
Since Brownian motion is a Gaussian process, conditioning it on fixed endpoints results in 
a Brownian bridge whose conditional expectation evolves linearly in time. 
In the IMF framework, we construct such Brownian bridges between sampled endpoint pairs 
and use their drift as regression targets. Specifically, the forward drift function 
$v_\theta$ is trained to approximate the drift of these endpoint-conditioned Brownian bridges. 
The parameters $\theta$ is obtained by solving
\begin{align}
\theta^\star = \argmin_\theta\! \int_0^1 \!\!\mathbb{E}_{\Pi_{k,1}}\!\!
\left[ \left\|
\frac{\mathbf{p}_1\!-\!\mathbf{p}_k^+}{1\!-\!k}
\!-\! v_\theta(k, \mathbf{p}_k^+)
\right\|^2 \right] \mathrm{d}k,
\label{eq:forward_opt}
\end{align}
where $\mathbf{p}_k^+$ denotes an intermediate sample drawn from a reference Brownian bridge conditioned on endpoints $(\hat{\mathbf{p}}_0, \mathbf{p}_1)$.
Here, $\mathbf{p}_1\sim\Pi_1$ is a real target-domain sample, 
and $\hat{\mathbf{p}}_0$ is the source-domain endpoint generated by the backward process from a previous backward iteration.
Under this reference bridge, $\mathbf{p}_k^+$ admits the closed-form interpolation
\begin{equation}
\mathbf{p}_k^{+} = (1-k)\hat{\mathbf{p}}_0 + k\mathbf{p}_1 + \sigma_0 \sqrt{k(1-k)}\ \mathbf{z},
\label{eq:interp_forward}
\end{equation}
where $\mathbf{z}\sim\mathcal{N}(0,I)$. 
\paragraph{Backward Process}
Similarly, the backward process starts from a target-domain sample 
$\mathbf{p}_0^T\sim\Pi_1$ and evolves backward in time according to
\begin{equation}
\mathrm{d}\mathbf{p}_k^T = v_{\phi^\star}(1-k, \mathbf{p}_k^T)\,\mathrm{d}k + \sigma_0\,\mathrm{d}\mathbf{B}_k.
\label{eq:backward}
\end{equation}
The forward and backward SDEs correspond to two time-direction parameterizations of the same Schrödinger bridge process. 
Under the Schrödinger bridge law, their intermediate states satisfy the distributional identity
$\mathbf{p}_k^S \stackrel{d}{=} \mathbf{p}_{1-k}^T$.
Accordingly, the backward drift function $v_\phi$
is trained by regressing against the drift of Brownian bridges constructed between sampled endpoint pairs 
as follows:
\begin{align}
	\!\!\!
\phi^\star \!=\! \argmin_\phi \! \int_0^1 
\!\!\!\mathbb{E}_{\Pi_{0,k}}\!\!
\left[ \left\|
\frac{\mathbf{p}_0\!-\!\mathbf{p}_k^-}{k}
\!-\! v_\phi(k, \mathbf{p}_k^-)
\right\|^2 \right]\! \mathrm{d}k,
\label{eq:backward_opt}
\end{align}
where $\mathbf{p}_k^-$ is sampled from a reference Brownian bridge conditioned on 
endpoints $(\mathbf{p}_0, \hat{\mathbf{p}}_1)$, with $\mathbf{p}_0 \sim \Pi_0$ and 
$\hat{\mathbf{p}}_1$ generated by the forward process from a previous forward iteration. 
Under this reference bridge, the intermediate sample admits the closed-form expression
\begin{align}
\mathbf{p}_k^{-} =(1-k)\mathbf{p}_0 + k\hat{\mathbf{p}}_1 + \sigma_0 \sqrt{k(1-k)}\ \mathbf{z}.
\label{eq:interp_backward}
\end{align}

By alternately optimizing  $\theta$ and $\phi$, 
we obtain forward and backward drift functions that jointly approximate the optimal 
Schrödinger bridge between between $\Pi_0$ and $\Pi_1$, 
which we subsequently use to perform transition dynamics alignment between the two domains.

\paragraph{Dynamics Alignment}
Once the drift functions are learned, we can translate source-domain transitions into target-domain transitions 
by deploying the forward process. 
However, solving the continuous-time SDEs exactly is typically intractable. 
In practice, we approximate the diffusion process using a discretized Langevin-like scheme, commonly the Euler–Maruyama method~\citep{shi2023diffusion,song2020score}. 
This allows us to denoise
over $K$ discrete steps with fixed time interval $\Delta k=\frac{1}{K}$. At each timestep $k$, the transition vector is updated as:
\begin{equation}
	\mathbf{p}_{k+1} = \mathbf{p}_k + v_{\theta^\star}(k\cdot\Delta k, \mathbf{p}_k)\Delta k 
	+ \sigma_0\sqrt{\Delta k}\boldsymbol{\xi}_k,
\end{equation}
where $\boldsymbol{\xi}_k\sim\mathcal{N}(0,I)$.
After $K$ steps, we obtain the translated transition $\mathbf{p}^T\approx\mathbf{p}_K$ that mimics the target domain. A similar procedure can be applied in the reverse direction using the backward drift function $v_{\phi^\star}$.

This procedure enables us to translate 
source-domain transitions into target-style transitions by sampling from the learned bridge distribution $\mathbb{P}^\star$, without 
access to the target environment. 
For notational convenience, we define $\mathrm{DSB}_{S\to T}$ as the forward translation operator and $\mathrm{DSB}_{T\to S}$ as the reverse translation operator. Thus, given a source transition $\mathbf{p}^S=(s_t,a_t,s_{t+1})$, we can generate the corresponding target-style next state $\tilde{s}_{t+1}$ via forward translation:
\begin{equation}
\tilde{s}_{t+1} \sim \mathrm{DSB}_{S \to T}(s_t, a_t, s_{t+1}).
\end{equation}
\subsection{Reward Modulation}
With the DSB-aligned transition dynamics, we are able to obtain target-style transitions to mitigate the dynamics gap between the source and target domains. However, learning a target-domain policy remains infeasible without a reward function from the target environment, as the offline target dataset $\mathcal{D}_T$ does not contain any reward information.

In modern RL literature, rewards are often modeled as a function of the state-action pair, $R(s_t,a_t)$. 
A straightforward approach is to reuse the reward function from the source environment. 
However, in many settings, particularly those involving goal-reaching or control tasks, the reward is more naturally tied to the outcome of the transition, i.e., the next state $s_{t+1}$. In fact, some formulations treat the reward as a function of $s_{t+1}$ alone~\citep{sutton1998reinforcement}. Consequently, applying the source reward function directly in the target domain can be misleading due to the misalignment in transition dynamics.

\paragraph{Transition-Aware Reward} 
To address the problem, we propose a transition-aware reward function $R(s_t,s_{t+1})$, which depends on both the current state and the resulting next state, but not on the action taken. This formulation captures the essence of transition outcome and is agnostic to the specific control signal that causes the transition.

We train the transition-aware reward model using the 
dataset $\mathcal{D}_S=\{(s_t,s_{t+1},r_t)\}$
collected from the source domain environment, 
which includes both transitions and ground-truth rewards. 
The reward model is learned by minimizing the following mean squared error (MSE):
\begin{equation}
\mathcal{L}_{\text{reward}} = \mathbb{E}_{(s_t, s_{t+1}, r_t) \sim \mathcal{D}_S} \left[ \left(R(s_t, s_{t+1}) - r_t \right)^2 \right].
\label{eq:reward}
\end{equation}
Afterward, we can apply this reward model to the translated target-style transitions
to produce target-oriented rewards. 
Specifically, for each source transition $(s_t, a_t, s_{t+1}, r_t)$, we first compute the target-aligned next state $\tilde{s}_{t+1}\sim \mathrm{DSB}_{S\to T}(s_{t+1}|s_t, a_t)$, and then modulate the reward as:
\begin{align}
\tilde{r}_t = R(s_t, \tilde{s}_{t+1}).
\end{align}
This reward modulation mechanism ensures consistency between the reward and the target-domain dynamics, while relying solely on source-domain interactions and supervision.


\subsection{Target-Oriented Policy Learning}
In cross-domain RL, the offline target dataset $\mathcal{D}_T$ is typically both limited in size and lacks reward supervision, making it unsuitable for direct policy optimization. To overcome this challenge, 
our proposed BDGxRL framework  
leverages modulated source-environment interactions to learn a target-oriented policy, based on DSB-driven dynamics alignment and transition-aware reward modulation. 
The goal is to train a policy solely within the source domain that generalizes well to the target environment, despite the dynamics mismatch and absence of target rewards.

Specifically, our proposed BDGxRL framework has two two phases--offline pretraining phase and online policy learning phase. 
During offline pretraining, 
we first collect a dataset $\mathcal{D}_S$ by sampling rollouts from the source environment. 
Using $\mathcal{D}_S$ and the target demonstrations $\mathcal{D}_T$, 
we then train a DSB model to align the source and target dynamics, 
and concurrently fit a transition-aware reward model.

During online policy learning, 
at each online interaction step in the source environment, 
we sample an action $a_t$ from the current policy $\pi$. Executing $a_t$ yields the next state $s_{t+1}$ under the source dynamics. We then apply the learned DSB model to translate the transition into a target-style transition $\tilde{s}_{t+1} \sim \text{DSB}_{S \rightarrow T}(s_t, a_t, s_{t+1})$. Given this translated next state, we use the learned reward model to compute a modulated reward $\tilde{r}_t = R(s_t, \tilde{s}_{t+1})$. The transition $(s_t, a_t, \tilde{r}_t, \tilde{s}_{t+1})$ is added to a replay buffer $\mathcal{B}$.
We use an off-policy RL algorithm, such as Soft Actor-Critic (SAC)~\citep{haarnoja2018sac2,haarnoja2018sac1}, 
to optimize the policy using the transitions stored 
in $\mathcal{B}$. 
This enables learning a policy that maximizes target-style rewards 
while interacting exclusively with the source environment.

To further improve sample efficiency and leverage the expert demonstrations from the target domain, we initialize the policy using imitation learning, for example via Behavior Cloning (BC)~\citep{torabi2018behavioral}. The imitation loss is defined as:
\begin{equation}
\mathcal{L}_\text{IL}=\mathbb{E}_{(s,a)\sim\mathcal{D}_T}[-\log\pi(a|s)].
\end{equation}
We then incorporate the imitation loss as a regularization term during online RL optimization. 
The final policy objective becomes:
\begin{equation}
\mathcal{L}_\text{total}=\mathcal{L}_\text{RL}+\alpha\cdot\mathcal{L}_\text{IL},
\end{equation}
where $\mathcal{L}_\text{RL}$ denotes the standard RL loss,
and $\alpha$ is a balancing coefficient between exploration and imitation. 

Below we provide a theoretical analysis that shows 
the learned target-oriented policy can closely approximate the optimal policy 
under mild assumptions. 
\begin{theorem}[Policy Value Bound under DSB Translation]
Assume the reward is bounded by $R_{\max}$, and the discount factor satisfies $\gamma \in [0,1)$. 
Let $\pi$ be the policy learned using BDGxRL with 
DSB-based dynamics translation and reward modulation, 
and let $\pi^\star$ denote the optimal policy in the target environment. 
Then, when the number of IMF iterations 
for DSB training
is sufficiently large, the value difference between $\pi$ and $\pi^\star$ in the target MDP $\mathcal{M}_T$ is bounded as:
{\small
\begin{equation}
\Delta V := \left| V^{\pi^\star}_{\mathcal{M}_T}\! -\! V^{\pi}_{\mathcal{M}_\text{DSB}} \right|
\le \frac{\sqrt{2} R_{\max}\gamma}{(1 - \gamma)^2} \cdot \sqrt{\epsilon_m}
+ \frac{2 \sqrt{2} R_{\max}}{(1 - \gamma)^2} \cdot \sqrt{\epsilon_\pi},
\end{equation}}
where $\epsilon_m = \mathcal{O}\left( \frac{1}{N_S} + \frac{1}{N_T} \right)$
denotes the dynamics approximation error induced by DSB training
with $N_S$ = $|\mathcal{D}_S|$ and $N_T$ = $|\mathcal{D}_T|$, 
and $\epsilon_\pi$ captures the policy approximation error 
between the learned target-oriented policy and 
the optimal target-oriented policy. 
\end{theorem}
This theorem demonstrates the theoretical soundness of the proposed approach 
in bridging the dynamics mismatch and enabling reliable policy learning in the target domain.
\section{Experiments}
\begin{table*}[t]
\centering
\footnotesize
\setlength{\tabcolsep}{3pt}
\resizebox{\textwidth}{!}{
\begin{tabular}{lllcccccccc}
\toprule
\textbf{Task} & \textbf{Expert} & \textbf{Domain Gap} & \textbf{D4RL} & \textbf{xTED (IL)} & \textbf{xTED (CD)} & \textbf{DARA} & \textbf{GAIL} & \textbf{DARC} & \textbf{DARAIL} & \textbf{BDGxRL (Ours)} \\
\midrule
\multirow{9}{*}{HalfCheetah} 
& \multirow{3}{*}{Med} 
& Gravity     & $39.5\pm2.4$ & $40.6\pm2.0$ & $41.1\pm2.0$ & $40.6\pm1.9$ & $39.0\pm1.8$ & $40.1\pm1.6$ & $42.6\pm2.4$ & $\mathbf{46.3\pm2.7}$ \\
& & Friction    & $43.1\pm1.5$ & $43.2\pm3.0$ & $44.2\pm1.8$ & $42.3\pm1.7$ & $40.2\pm2.0$ & $41.7\pm1.9$ & $43.9\pm2.3$ & $\mathbf{47.8\pm2.9}$ \\
& & Thigh Size  & $39.5\pm2.4$ & $40.7\pm2.4$ & $41.5\pm1.6$ & $40.4\pm2.1$ & $39.5\pm2.2$ & $40.9\pm1.7$ & $42.7\pm2.1$ & $\mathbf{46.5\pm2.5}$ \\
\cmidrule{2-11}
& \multirow{3}{*}{Med-R} 
& Gravity     & $26.2\pm3.5$ & $30.7\pm1.9$ & $30.7\pm1.9$ & $27.2\pm3.2$ & $28.2\pm1.7$ & $30.5\pm2.2$ & $33.6\pm2.3$ & $\mathbf{37.8\pm2.8}$ \\
& & Friction    & $26.2\pm3.5$ & $31.8\pm3.1$ & $32.4\pm2.9$ & $23.2\pm4.7$ & $30.0\pm1.9$ & $31.3\pm2.5$ & $35.0\pm2.4$ & $\mathbf{39.5\pm2.6}$ \\
& & Thigh Size  & $26.2\pm3.5$ & $30.3\pm1.9$ & $32.9\pm2.7$ & $29.1\pm1.4$ & $28.9\pm2.1$ & $30.7\pm2.4$ & $34.4\pm2.1$ & $\mathbf{38.6\pm2.7}$ \\
\cmidrule{2-11}
& \multirow{3}{*}{Med-E} 
& Gravity     & $40.1\pm2.9$ & $45.3\pm3.8$ & $45.3\pm3.8$ & $39.7\pm3.8$ & $40.7\pm1.5$ & $42.4\pm8.9$ & $47.7\pm6.5$ & $\mathbf{53.2\pm3.0}$ \\
& & Friction    & $40.1\pm2.9$ & $45.0\pm3.0$ & $45.0\pm3.0$ & $41.4\pm3.5$ & $41.5\pm1.7$ & $44.1\pm5.3$ & $48.6\pm4.1$ & $\mathbf{55.1\pm3.5}$ \\
& & Thigh Size  & $40.1\pm2.9$ & $43.0\pm3.1$ & $45.1\pm3.3$ & $37.9\pm4.1$ & $46.8\pm2.3$ & $48.2\pm2.6$ & $51.0\pm2.9$ & $\mathbf{55.6\pm2.5}$ \\
\midrule
\multirow{9}{*}{Walker2d} 
& \multirow{3}{*}{Med} 
& Gravity     & $45.3\pm15.9$ & $58.2\pm11.7$ & $59.0\pm10.5$ & $53.2\pm11.1$ & $55.3\pm2.1$ & $56.6\pm2.6$ & $60.4\pm3.0$ & $\mathbf{63.9\pm3.2}$ \\
& & Friction    & $45.3\pm15.9$ & $54.3\pm10.6$ & $54.3\pm10.6$ & $51.1\pm11.4$ & $52.2\pm2.3$ & $54.5\pm2.4$ & $58.0\pm3.1$ & $\mathbf{61.6\pm3.4}$ \\
& & Thigh Size  & $45.3\pm15.9$ & $58.7\pm9.5$ & $58.3\pm9.8$ & $49.1\pm10.6$ & $53.1\pm2.5$ & $55.2\pm2.7$ & $59.2\pm2.8$ & $\mathbf{62.8\pm3.0}$ \\
\cmidrule{2-11}
& \multirow{3}{*}{Med-R} 
& Gravity     & $26.6\pm5.9$ & $31.8\pm4.7$ & $33.5\pm3.3$ & $13.2\pm5.7$ & $30.3\pm2.0$ & $32.8\pm2.1$ & $36.2\pm2.5$ & $\mathbf{39.8\pm3.1}$ \\
& & Friction    & $26.6\pm5.9$ & $30.3\pm2.4$ & $30.3\pm2.4$ & $28.3\pm9.0$ & $28.7\pm2.2$ & $31.5\pm2.0$ & $34.8\pm2.7$ & $\mathbf{38.1\pm3.2}$ \\
& & Thigh Size  & $26.6\pm5.9$ & $29.4\pm3.3$ & $29.4\pm3.3$ & $27.0\pm6.6$ & $28.0\pm2.3$ & $30.4\pm2.1$ & $33.9\pm2.6$ & $\mathbf{37.2\pm2.9}$ \\
\cmidrule{2-11}
& \multirow{3}{*}{Med-E} 
& Gravity     & $71.0\pm21.0$ & $82.9\pm11.1$ & $83.2\pm11.6$ & $73.4\pm14.3$ & $76.6\pm2.9$ & $79.9\pm16.7$ & $83.6\pm11.5$ & $\mathbf{87.0\pm4.9}$ \\
& & Friction    & $71.0\pm21.0$ & $74.0\pm24.5$ & $74.0\pm24.5$ & $78.8\pm20.2$ & $77.2\pm3.3$ & $80.1\pm17.9$ & $85.4\pm9.8$ & $\mathbf{89.5\pm5.0}$ \\
& & Thigh Size  & $71.0\pm21.0$ & $81.0\pm11.1$ & $81.0\pm11.1$ & $78.4\pm14.3$ & $75.3\pm2.5$ & $78.7\pm2.8$ & $83.0\pm3.2$ & $\mathbf{86.3\pm3.4}$ \\
\bottomrule
\end{tabular}
}
\caption{Experimental results across three domain gaps (gravity, friction, thigh size) and three demonstration levels (Medium, Medium-Replay, Med-Expert) on two MuJoCo environments, averaged over five runs.}
\label{tab:main_results}
\end{table*}
\subsection{Experimental Settings}
\paragraph{Experimental Environments} 
We evaluate our method in MuJoCo~\citep{todorov2012mujoco} physics-based robotic simulation environments, with offline demonstration datasets provided by the D4RL benchmark~\cite{fu2020d4rl}. The source domain is derived by modifying the physics parameters of the original MuJoCo environments, introducing three types of domain gaps: (1) 2× gravity, (2) 0.25×/0.5× friction, and (3) 2× thigh size of the robots. These modifications alter the transition dynamics between the source and target environments while leaving the state and action spaces unchanged. 
The experiments involve two tasks (HalfCheetah and Walker2d), three levels of target demonstration expertise (Medium, Medium-Replay, Medium-Expert), and the three types of domain gaps mentioned above. 
\paragraph{Comparison Methods}
We compare our method against several state-of-the-art cross-domain RL baselines, including xTED~\citep{niu2024xted} (combined with IQL~\citep{kostrikov2021offline} or DARA~\citep{liu2022dara}), DARA~\citep{liu2022dara}, DARC~\citep{eysenbach2020off}, DARAIL~\citep{guo2024off}, as well as the source behavior policy from the expert data in D4RL~\citep{fu2020d4rl}, and an imitation learning policy (GAIL)~\citep{ho2016generative} trained on the target domain.
\subsection{Experimental Results}
We report the experimental results using D4RL normalized scores~\citep{fu2020d4rl}, where episodic returns are scaled such that a random policy scores 0 and an expert policy scores 100. This score is a widely adopted metric for comparing performance across different tasks and methods. All results are summarized in Table~\ref{tab:main_results}.

Overall, our proposed BDGxRL consistently achieves the highest performance across all tasks and settings, demonstrating strong generalizability and adaptability under various dynamics gaps. In the HalfCheetah environment, BDGxRL surpasses all baselines across all domain gaps and expert levels. For example, under the Medium-Expert setting with the Gravity gap, BDGxRL reaches 53.2, significantly outperforming DARC (47.7) and DARAIL (51.0). Even under the challenging Medium-Replay setting, BDGxRL shows robust performance, achieving up to 39.5 (Friction) compared to 35.0 by DARAIL and 32.4 by xTED (CD).

A similar trend is observed in the Walker2d environment. BDGxRL outperforms prior methods in all domain gap settings. Under the Medium-Expert setting with Friction, it achieves 89.5, outperforming DARAIL (85.4) and GAIL (77.2). Notably, even when the baseline methods suffer from high variance or degraded performance in low-quality demonstrations (e.g., DARA in Med-R setting), BDGxRL maintains stable and superior results across all cases.

We also visualize the average scores of each method across all tasks in Figure~\ref{fig:overall_performance}, which summarizes the overall performance across diverse tasks, domain gaps, and demonstration levels. 
BDGxRL consistently outperforms all baselines, achieving the highest mean score among all methods. The noticeable gap between BDGxRL and the second-best method (DARAIL) highlights its strong generalization capability under dynamics mismatch. In contrast, baseline methods lacking explicit dynamics alignment, such as D4RL and GAIL, consistently underperform.

These results demonstrate the effectiveness of our proposed dynamics alignment and reward modulation strategies. By jointly leveraging online source rollouts and offline target demonstrations, BDGxRL bridges the dynamics gap while preserving policy optimality, establishing a new state-of-the-art in cross-domain RL.
\begin{figure}[t]
\centering
\includegraphics[width=.85\linewidth]{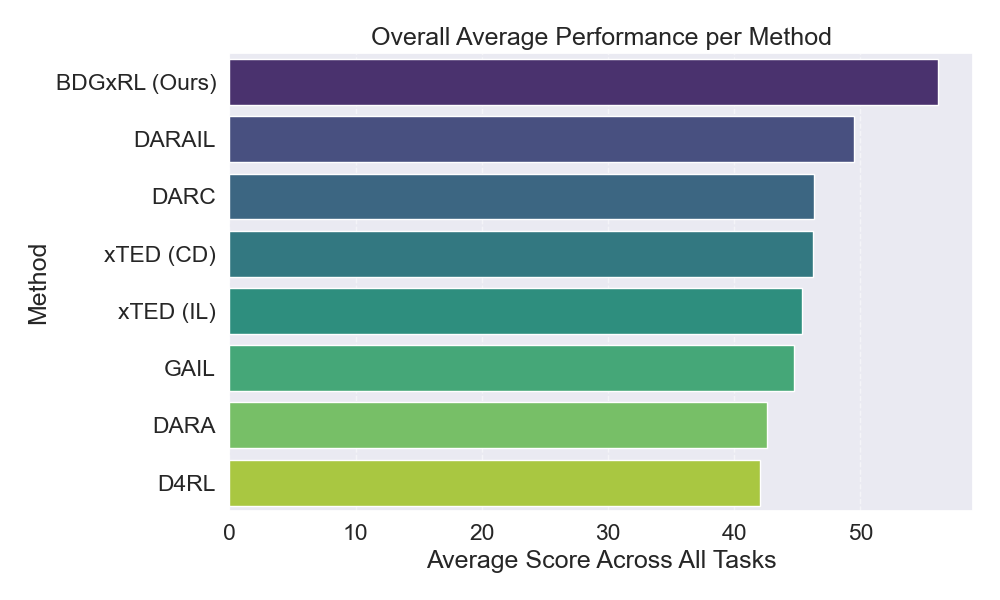}
\caption{Overall average performance of each method across all tasks, domain gaps, and demonstration levels.}
\label{fig:overall_performance}
\vskip -.15in
\end{figure}
\subsection{Ablation Study}
\begin{table}[t]
\centering
\footnotesize
\setlength{\tabcolsep}{2pt}
\resizebox{\columnwidth}{!}{
\begin{tabular}{llcccc}
\toprule
	\textbf{Task} & \textbf{\makecell{Domain\\ Gap}} & \textbf{BDGxRL} & \textbf{w/o IL} & \textbf{w/o RM} & \textbf{\makecell{w/o\\ Alignment}} \\
\midrule
\multirow{3}{*}{HalfCheetah} 
	& Gravity     & $46.3_{ \pm 2.7}$ & $39.5_{ \pm 3.5}$ & $42.1_{ \pm 3.1}$ & $33.8_{\pm 5.2}$ \\
	& Friction    & $47.8_{ \pm 2.9}$ & $40.3_{ \pm 3.3}$ & $43.2_{\pm 2.9}$ & $35.4_{ \pm 4.8}$ \\
	& Thigh Size  & $46.5_{ \pm 2.5}$ & $40.1_{ \pm 3.0}$ & $42.8_{ \pm 2.8}$ & $34.9_{ \pm 5.1}$ \\
\midrule
\multirow{3}{*}{Walker2d} 
	& Gravity     & $63.9_{ \pm 3.2}$ & $55.5_{\pm 4.1}$ & $58.2_{ \pm 3.3}$ & $44.7_{\pm 6.7}$ \\
	& Friction    & $61.6_{\pm 3.4}$ & $53.3_{\pm 3.9}$ & $56.1_{\pm 3.1}$ & $43.2_{\pm 5.8}$ \\
	& Thigh Size  & $62.8_{\pm 3.0}$ & $54.1_{\pm 4.0}$ & $56.8_{\pm 3.2}$ & $43.9_{\pm 6.0}$ \\
\bottomrule
\end{tabular}
	}
\caption{Ablation results under the Medium expert level across three domain gaps on two MuJoCo environments.} 
\label{tab:ablation_selected}
\vskip -.1in
\end{table}
To further evaluate the contribution of each component in our method, we conduct an ablation study by comparing the full BDGxRL framework with three ablated variants: (1) w/o IL, which removes imitation learning for policy initialization;
(2) w/o RM, which disables reward modulation and directly uses the source domain reward model; 
and (3) w/o Alignment, which removes the transition modulation step, relying solely on source domain dynamics without applying diffusion-based adaptation. The results are reported in Table~\ref{tab:ablation_selected}, focusing on the Medium expert level.

The results show that removing transition alignment
leads to the most significant performance drop across both tasks and all domain gaps, highlighting its critical role in addressing dynamics mismatch. Excluding imitation learning also results in substantial performance degradation, particularly on Walker2d, emphasizing the importance of leveraging offline target demonstrations. Reward modulation has a comparatively smaller but consistent impact on performance. Overall, these findings confirm that all three components: transition alignment, imitation learning, and reward modulation, play essential and complementary roles in achieving robust cross-domain policy learning.

%
\section{Conclusion}
In this work, 
we proposed BDGxRL, a novel framework that performs transition dynamics alignment via DSB
and reward modulation for cross-domain RL. 
It enables target-oriented online policy learning
in a source environment using only offline expert demonstrations without reward signals from the target domain. 
We also provided a theoretical 
bound on the learned policy value gap. 
Experimental results on MuJoCo benchmarks across various domain gaps and demonstration levels 
show that BDGxRL consistently outperforms prior state-of-the-art approaches, demonstrating its robustness and 
effectiveness in bridging dynamics discrepancies for cross-domain RL.

\bibliography{aaai26}
\clearpage
%
\appendix
\section{Training Algorithm of BDGxRL}

The full training procedure of the proposed BDGxRL is outlined in Algorithm~\ref{alg:target-policy}.

\begin{algorithm}[t]
\caption{Target-Oriented Policy Learning of BDGxRL}
\label{alg:target-policy}
\begin{algorithmic}[1]
\STATE \textbf{Input:} Source environment $E_S$, offline target dataset $\mathcal{D}_T$, initial policy $\pi(a_t|s_t)$, replay buffer $\mathcal{B}$
\vspace{3pt}
\STATE \% \textbf{Offline Pretraining Phase:}
\STATE Collect source dataset $\mathcal{D}_S$ by interacting with $E_S$.
\STATE Train the transition-aware reward model $R(s_t, s_{t+1})$ on $\mathcal{D}_S$ using Eq. (10).
\STATE 
Train the DSB model on $\mathcal{D}_S$ and $\mathcal{D}_T$ using Eq. (4)(5) and Eq.(6)(7).
\STATE Pretrain policy $\pi(a_t|s_t)$ via behavior cloning on $\mathcal{D}_T$.
\vspace{3pt}
\STATE \% \textbf{Online Policy Learning Phase:}
\WHILE{not converged}
\STATE Reset environment $E_S$ and observe initial state $s_0$.
\FOR{each timestep $t$}
\STATE Sample action $a_t \sim \pi(\cdot \mid s_t)$.
\STATE Execute $a_t$ to obtain next state $s_{t+1}$ and reward $r_t$.
\STATE Translate state: $\tilde{s}_{t+1} = \mathrm{DSB}_{S \to T}(s_t, a_t, s_{t+1})$.
\STATE Compute modulated reward: $\tilde{r}_t = R(s_t, \tilde{s}_{t+1})$.
\STATE Append $(s_t, a_t, s_{t+1}, r_t)$ to $\mathcal{D}_S$ for model update.
\STATE Store transition $(s_t, a_t, \tilde{r}_t, \tilde{s}_{t+1})$ in replay buffer $\mathcal{B}$.
\STATE Set $s_{t+1}=\tilde{s}_{t+1}$.
\ENDFOR
\STATE Update policy $\pi(a_t|s_t)$ using SAC with buffer $\mathcal{B}$.
\STATE Periodically update DSB and reward models using the augmented $\mathcal{D}_S$.
\ENDWHILE
\end{algorithmic}
\end{algorithm}
%

\section{Full Ablation Results}
\begin{table*}[t]
\centering
\footnotesize
\setlength{\tabcolsep}{4pt}
{
\begin{tabular}{lllcccc}
\toprule
\textbf{Task} & \textbf{Expert} & \textbf{Domain Gap} & \textbf{BDGxRL (Ours)} & \textbf{w/o IL} & \textbf{w/o RM} & \textbf{w/o Alignment} \\
\midrule
\multirow{9}{*}{HalfCheetah}
& \multirow{3}{*}{Med}
& Gravity     & $46.3 \pm 2.7$ & $39.5 \pm 3.5$ & $42.1 \pm 3.1$ & $33.8 \pm 5.2$ \\
& & Friction    & $47.8 \pm 2.9$ & $40.3 \pm 3.3$ & $43.2 \pm 2.9$ & $35.4 \pm 4.8$ \\
& & Thigh Size  & $46.5 \pm 2.5$ & $40.1 \pm 3.0$ & $42.8 \pm 2.8$ & $34.9 \pm 5.1$ \\
\cmidrule{2-7}
& \multirow{3}{*}{Med-R}
& Gravity     & $37.8 \pm 2.8$ & $31.0 \pm 4.0$ & $33.2 \pm 3.7$ & $25.6 \pm 6.1$ \\
& & Friction    & $39.5 \pm 2.6$ & $32.5 \pm 3.2$ & $34.4 \pm 3.4$ & $27.2 \pm 5.3$ \\
& & Thigh Size  & $38.6 \pm 2.7$ & $31.8 \pm 3.0$ & $33.9 \pm 3.2$ & $26.4 \pm 5.0$ \\
\cmidrule{2-7}
& \multirow{3}{*}{Med-E}
& Gravity     & $53.2 \pm 3.0$ & $45.2 \pm 4.2$ & $48.8 \pm 3.6$ & $37.5 \pm 6.4$ \\
& & Friction    & $55.1 \pm 3.5$ & $46.3 \pm 4.0$ & $50.1 \pm 3.7$ & $39.3 \pm 6.0$ \\
& & Thigh Size  & $55.6 \pm 2.5$ & $46.9 \pm 3.8$ & $50.4 \pm 3.6$ & $40.7 \pm 5.8$ \\
\midrule
\multirow{9}{*}{Walker2d}
& \multirow{3}{*}{Med}
& Gravity     & $63.9 \pm 3.2$ & $55.5 \pm 4.1$ & $58.2 \pm 3.3$ & $44.7 \pm 6.7$ \\
& & Friction    & $61.6 \pm 3.4$ & $53.3 \pm 3.9$ & $56.1 \pm 3.1$ & $43.2 \pm 5.8$ \\
& & Thigh Size  & $62.8 \pm 3.0$ & $54.1 \pm 4.0$ & $56.8 \pm 3.2$ & $43.9 \pm 6.2$ \\
\cmidrule{2-7}
& \multirow{3}{*}{Med-R}
& Gravity     & $39.8 \pm 3.1$ & $33.3 \pm 3.8$ & $35.5 \pm 3.6$ & $28.2 \pm 5.4$ \\
& & Friction    & $38.1 \pm 3.2$ & $32.1 \pm 3.7$ & $34.0 \pm 3.3$ & $27.0 \pm 4.9$ \\
& & Thigh Size  & $37.2 \pm 2.9$ & $31.3 \pm 3.5$ & $33.2 \pm 3.2$ & $26.2 \pm 4.7$ \\
\cmidrule{2-7}
& \multirow{3}{*}{Med-E}
& Gravity     & $87.0 \pm 4.9$ & $75.1 \pm 5.6$ & $80.3 \pm 4.2$ & $62.4 \pm 7.5$ \\
& & Friction    & $89.5 \pm 5.0$ & $77.3 \pm 5.4$ & $82.0 \pm 4.3$ & $64.2 \pm 7.2$ \\
& & Thigh Size  & $86.3 \pm 3.4$ & $74.8 \pm 5.3$ & $79.4 \pm 4.1$ & $61.8 \pm 6.8$ \\
\bottomrule
\end{tabular}
}
\caption{Full ablation results across three domain gaps (gravity, friction, thigh size) and three demonstration levels (Medium, Medium-Replay, Med-Expert) on two MuJoCo environments, averaged over three runs.}
\label{tab:ablation_full}
\end{table*}
Table~\ref{tab:ablation_full} reports the full ablation results across all expert levels (Medium, Medium-Replay, and Medium-Expert) and domain gaps (gravity, friction, and thigh size) on two MuJoCo tasks. The results consistently highlight the importance of each component within BDGxRL.

Across all settings, removing transition alignment (w/o Alignment) leads to the most pronounced performance degradation, particularly in high-performing regimes such as Walker2d (Med-E), where the score drops from 87.0 to 62.4 (gravity) and from 89.5 to 64.2 (friction). This underscores the essential role of DSB-based transition alignment in handling dynamics mismatch.
Excluding imitation learning (w/o IL) also causes significant drops in performance, especially in lower-quality data regimes such as HalfCheetah (Med-R), where training from scratch limits policy effectiveness. This demonstrates the value of leveraging offline target demonstrations for policy initialization under domain shifts.
Disabling reward modulation (w/o RM) results in moderate yet consistent performance declines. For example, in Walker2d (Med-R), performance on the friction setting drops from 38.1 to 32.1. Although its effect is less drastic than the other components, reward modulation still contributes to improved alignment of cross-domain rewards.

In summary, the full ablation results confirm the complementary nature of BDGxRL’s three key components. Each plays a distinct role in enabling robust cross-domain policy learning across a wide range of domain gaps and data conditions.
\section*{Theoretical Analysis of BDGxRL}
In this section, we provide a theoretical analysis of BDGxRL, based on policy evaluation error bounds, to establish a theoretical guarantee for the effectiveness of the policy learned via BDGxRL. Specifically, we prove Theorem 1, which shows that the performance of the learned BDGxRL policy is bounded within a specific error margin under mild assumptions.

To prove Theorem 1, we begin by analyzing the total sources of error in policy learning. These errors primarily arise from two factors:
\begin{itemize}
\item Transition model error: due to the inaccuracy of the learned transition model via DSB.
\item Policy optimization error: resulting from the suboptimality of policy learning using deep neural networks.
\end{itemize}
Formally, let $\mathcal{M}_T$ denote the true target MDP, $\mathcal{M}_\text{DSB}$ the learned transition model obtained via DSB, $\pi^\star$ the optimal policy in the true MDP, and $\pi$ the target-oriented policy learned by BDGxRL. Then the total policy evaluation error can be decomposed as:
\begin{equation}
\left| V^{\pi^\star}_{\mathcal{M}_T} - V^{\pi}_{\mathcal{M}_\text{DSB}} \right|
\leq 
\underbrace{\left| V^{\pi^\star}_{\mathcal{M}_T} - V^{\pi^\star}_{\mathcal{M}_\text{DSB}} \right|}_{\text{transition model error}} +
\underbrace{\left| V^{\pi^\star}_{\mathcal{M}_\text{DSB}} - V^{\pi}_{\mathcal{M}_\text{DSB}} \right|}_{\text{policy optimization error}}.
\end{equation}
We now show that, under mild and reasonable assumptions, both error terms can be explicitly bounded.
\begin{assumption}
\label{asm:assumption1}
Let $N_\text{DSB}$ denote the number of IMF iterations used in solving the DSB, $N_S$ the size of the source domain dataset, and $N_T$ the size of the target domain dataset. The reward function is bounded by $R_{\max}$, and the discount factor satisfies $\gamma\in[0, 1)$. We assume that the DSB is correctly implemented with a score-based model and trained with a sufficiently large number of IMF iterations. The target dataset size $N_T = |\mathcal{D}_T|$ is fixed by offline demonstrations, while the source dataset size $N_S = |\mathcal{D}_S|$ grows during training via data collection in the source domain.
\end{assumption}
Under Assumption~\ref{asm:assumption1}, the DSB-based dynamics alignment is assumed to produce accurate transitions when the number of IMF iterations $N_\text{DSB}$ is large enough, and the score-based model is trained using both source and target data. We now present a bound on the transition model error.
\begin{lemma}[Transition Model Error Bound]
\label{lem:trans}
Under Assumption~\ref{asm:assumption1}, the transition model error is bounded as:
\begin{equation}
\label{eqa:lemma1}
\left| V^{\pi^\star}_{\mathcal{M}_T} - V^{\pi^\star}_{\mathcal{M}_\text{DSB}} \right|
\leq 
\frac{\sqrt{2} R_{\max} \gamma}{(1 - \gamma)^2} \cdot \sqrt{\epsilon_m},
\end{equation}
where the transition model error $\epsilon_m$ decomposes as:
\begin{equation}
\epsilon_m \approx \mathcal{O} \left( \frac{1}{N_S} + \frac{1}{N_T} \right).
\end{equation}
The transition model error $\epsilon_m$ decreases as the source data size increases during policy training, and is lower bounded by the target domain demonstration size as $\mathcal{O}(\frac{1}{N_T})$.
\end{lemma}
\begin{proof}
Based on the previous conclusion from~\citep{xu2020error}, when learning a policy in an imitation environment with transition model errors, and assuming that the learned transition model is close to the true transition dynamics in the sense that:
\begin{equation}
\mathbb{E}_{(s,a)\sim\rho^\pi} \left[ \mathrm{KL}\left( \mathcal{T}_T(\cdot|s,a) \parallel \mathcal{T}_\text{DSB}(\cdot|s,a) \right) \right] \leq \epsilon_m,
\end{equation}
then the transition model error is bounded as:
\begin{equation}
\left| V^{\pi^\star}_{\mathcal{M}_T} - V^{\pi^\star}_{\mathcal{M}_\text{DSB}} \right|
\leq
\frac{\sqrt{2} R_{\max} \gamma}{(1 - \gamma)^2} \cdot \sqrt{\epsilon_m}.
\end{equation}
\paragraph{Decomposition of $\epsilon_m$}
In our context, the error $\epsilon_m$ in learning the transition model arises from two sources: (1) the approximation error of the score-based model used in DSB, and (2) the optimization error in solving the DSB problem. We write:
\begin{equation}
\epsilon_m = \underbrace{\epsilon_{\text{sbm}}}_{\text{score-based model error}} + \underbrace{\epsilon_{\text{dsb}}}_{\text{DSB solver error}}.
\end{equation}
\paragraph{Score-based model error $\epsilon_{\text{sbm}}$}
According to~\citep{zhu2023sample}, the estimation error of score-based models approximately scales inversely with the amount of training data. Suppose the dataset size is $N$, then the estimation error is bounded by $\mathcal{O}(\frac{1}{N})$.
In our setting, the DSB model is trained bidirectionally, which effectively doubles the learning error. Let $N_S$ denote the number of transitions collected from the source domain and $N_T$ denote the number of offline demonstrations from the target domain. Then the total score-based model error is bounded by:
\begin{equation}
\epsilon_{\text{sbm}} = \mathcal{O}\left( \frac{1}{N_S} + \frac{1}{N_T} \right).
\end{equation}
Since transitions from the source domain are continuously collected during online training, the term $\mathcal{O}(\frac{1}{N_S})$ keeps decreasing over time. With sufficiently long training, the score-based model error $\epsilon_{\text{sbm}}$ will be dominated by the fixed-size offline demonstration data from the target domain.
\paragraph{DSB error $\epsilon_{\text{dsb}}$}
The DSB solver error $\epsilon_{\text{dsb}}$ originates from imperfect optimization of the DSB objective relative to the true target transition distribution. Let $\mathcal{T}_T$ denote the ground-truth transition distribution in the target environment, and let $\mathcal{T}^n$ denote the transition distribution induced by the DSB after $n$ iterations of IPF or IMF. Then:
\begin{equation}
\epsilon_{\text{dsb}} = \mathrm{KL}(\mathcal{T}^n \,\|\, \mathcal{T}_T).
\end{equation}
As discussed in the preliminary section, the residual error of DSB vanishes asymptotically as the number of IPF/IMF iterations increases. That is,
\begin{equation}
\lim_{n \to \infty}\epsilon_{\text{dsb}}= \lim_{n\to\infty}\mathrm{KL}(\mathcal{T}^n \,\|\, \mathcal{T}_T)= 0.
\end{equation}
Therefore, with sufficiently many iterations, the DSB optimization error $\epsilon_{\text{dsb}}$ can be made arbitrarily small.
Under Assumption~\ref{asm:assumption1}, we have shown that the transition model error term $\left| V^\pi_{\mathcal{M}_T} - V^\pi_{\mathcal{M}_\text{DSB}} \right|$ is bounded by Eq.~\eqref{eqa:lemma1}, with the overall error $\epsilon_m$ decomposed into two components, both of which diminish with sufficient data and training.
\end{proof}

\begin{theorem}[Policy Value Bound under DSB Translation]
Assume the reward is bounded by $R_{\max}$, and the discount factor satisfies $\gamma \in [0,1)$. 
Let $\pi$ be the policy learned using BDGxRL with 
DSB-based dynamics translation and reward modulation, 
and let $\pi^\star$ denote the optimal policy in the target environment. 
Then, when the number of IMF iterations 
for DSB training
is sufficiently large, the value difference between $\pi$ and $\pi^\star$ in the target MDP $\mathcal{M}_T$ is bounded as:
{\small
\begin{equation}
\Delta V := \left| V^{\pi^\star}_{\mathcal{M}_T}\! -\! V^{\pi}_{\mathcal{M}_\text{DSB}} \right|
\le \frac{\sqrt{2} R_{\max}\gamma}{(1 - \gamma)^2} \cdot \sqrt{\epsilon_m}
+ \frac{2 \sqrt{2} R_{\max}}{(1 - \gamma)^2} \cdot \sqrt{\epsilon_\pi},
\end{equation}}
where $\epsilon_m = \mathcal{O}\left( \frac{1}{N_S} + \frac{1}{N_T} \right)$
denotes the dynamics approximation error induced by DSB training
with $N_S$ = $|\mathcal{D}_S|$ and $N_T$ = $|\mathcal{D}_T|$, 
and $\epsilon_\pi$ captures the policy approximation error 
between the learned target-oriented policy and the optimal target-oriented policy.
\end{theorem}
\begin{proof}
Based on Theorem 1 of~\citep{xu2020error}, 
under Assumption~\ref{asm:assumption1}, for any policy $\pi$ satisfying $\mathrm{KL}(\pi(\cdot|s)\parallel\pi^\star(\cdot|s)) \leq \epsilon_\pi$, the policy optimization error is bounded as:
\begin{equation}
\left| V^{\pi^\star}_{\mathcal{M}_\text{DSB}} - V^{\pi}_{\mathcal{M}_\text{DSB}} \right|\leq
\frac{2 \sqrt{2} R_{\max}}{(1 - \gamma)^2} \cdot \sqrt{\epsilon_\pi}.
\end{equation}
Combining this with the transition model error $\left| V^\pi_{\mathcal{M}_T} - V^\pi_{\mathcal{M}_\text{DSB}} \right|$ established in Lemma~\ref{lem:trans}, the total value difference between the learned target-oriented policy $\pi$ and the optimal target policy $\pi^\star$ is bounded by:
\begin{equation}
\small
\Delta V := \left| V^{\pi^\star}_{\mathcal{M}_T} - V^{\pi}_{\mathcal{M}_\text{DSB}} \right|
\le \frac{\sqrt{2} R_{\max}\gamma}{(1 - \gamma)^2} \cdot \sqrt{\epsilon_m}
+ \frac{2 \sqrt{2} R_{\max}}{(1 - \gamma)^2} \cdot \sqrt{\epsilon_\pi}.
\end{equation}
\end{proof}
This result demonstrates that the proposed BDGxRL framework can effectively learn a target-oriented policy that achieves near-optimal performance in the target environment.
\section{Implementation Details} 
We conduct our experiments on MuJoCo~\citep{todorov2012mujoco} environments from the xTED benchmark~\citep{niu2024xted}, utilizing online interactions in the modified source domain and offline data from D4RL (medium-expert) as the target domain. Each task contains 20,000 offline transitions from demonstration data in the target domain. The three target environments are adapted from standard MuJoCo tasks with the following modifications:
(1) Gravity: doubling the gravitational acceleration to alter the simulation dynamics;
(2) Friction: reducing the friction coefficient to 0.25× or 0.5×, making it harder for the agent to maintain balance;
(3) Thigh Size: increasing thigh size by 2× to introduce morphological discrepancies in embodiment.
For dynamics alignment, we adopt the official implementation of Diffusion Schrödinger Bridge Matching (DSBM) and Iterative Markovian Fitting (IMF) from~\citep{shi2023diffusion} as the backbone of our DSB-based approach. To modulate rewards, we train a three-layer MLP reward model using source domain rollouts. Target-oriented policy learning is performed using SAC, initialized via behavior cloning (BC) on the target dataset. A KL regularization term with coefficient $\alpha = 0.1$ is applied to encourage consistency with the imitation policy. Some experimental environments are also partially adopted from Off-Dynamics Reinforcement Learning (ODRL)~\citep{lyu2024odrl}.
%

%
\end{document}